\documentclass{article}
\usepackage{iclr2027_conference,times}

\usepackage{hyperref}
\usepackage{url}

\usepackage{amsmath,amssymb,amsthm}
\usepackage{booktabs}
\usepackage{graphicx}
\usepackage{float}
\usepackage{caption}
\usepackage{algorithm}
\usepackage{algpseudocode}
\usepackage[table]{xcolor}
\usepackage{tikz}
\usetikzlibrary{arrows.meta,matrix,positioning}
\usepackage{microtype}

\newcommand{\E}{\mathbb E}

\newtheorem{proposition}{Proposition}
\newcommand{\w}{\mathbf{w}}

\iclrfinalcopy

\title{Beyond One Epoch: Uncertainty-Weighted Sensitivity Regularization for Recommendation Models}

\author{Richard Lettich \\
Meta Platforms, Inc. \\
Menlo Park, CA, USA \\
\texttt{rlett@meta.com} \\
\And
Shagun Gupta \\
Meta Platforms, Inc. \\
Menlo Park, CA, USA}

\begin{document}

\maketitle
\thispagestyle{fancy}
\lhead{Preprint. }

\begin{abstract}
Recommendation models with sparse embeddings and a shared consumer often exhibit the \emph{one-epoch phenomenon}: a second epoch lowers training loss while sharply degrading generalization.
We present a view based on the violation of the prequential principle. On the first epoch, an example's label has not affected the embedding rows used to score it.
On later epochs, those rows contain a displacement induced by the label's earlier update.
This creates an incentive for the shared consumer to exploit this displacement in subsequent epochs, which fails to generalize.
We call this \emph{self-influence asymmetry}.
Using an exact scalar model and local influence analysis, we connect this mismatch to the uncertainty in the embeddings and the consumer's incentive to exploit it in subsequent epochs.
We verify this hypothesis using an embedding-consumer-update intervention in deep recommendation models
and propose uncertainty-weighted sensitivity regularization (UWSR) which counteracts this mismatch by augmenting the loss function to penalize the consumer for relying on uncertain embeddings.
Unlike existing remedies, UWSR preserves the learned embeddings, and across three benchmarks, four-epoch UWSR reduces test cross-entropy by \(1.38\%\)–\(6.78\%\) and improves AUC by  0.0058–0.0231 relative to one-epoch training.

\end{abstract}
\section{Introduction}
\emph{Recommendation models} score the likelihood of a user interacting with an item, such as predicting click-through rates for displayed content. Under the hood, these systems map high-dimensional categorical inputs, like user and item IDs, to learned vector representations stored in embedding tables. As updates apply only to entities active in a given batch of examples during training, these embedding lookups are inherently sparse. A shared consumer, usually a dense neural network, then combines these retrieved vectors with other features to compute the final prediction. As modern systems scale to encompass billions of unique entities, these sparse embedding tables dominate the model parameter count compared to the consumer by several orders of magnitude.
\begin{figure}[H]
  \centering
  \includegraphics[width=\textwidth]{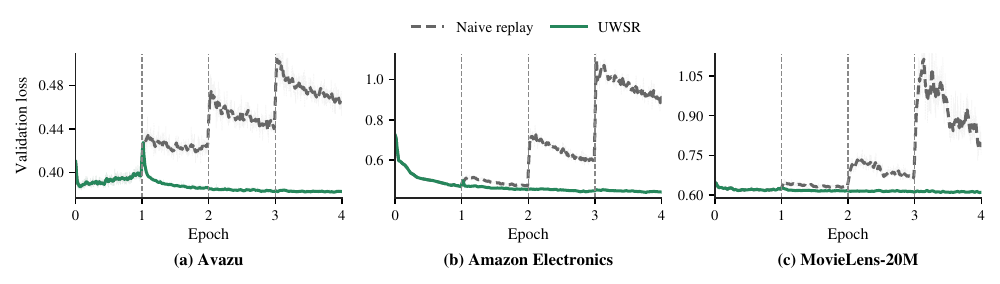}
  \caption{Validation Loss on Avazu, Amazon Electronics, and MovieLens-20M datasets comparing naive multi-epoch training with proposed UWSR method. Naive training displays one-epoch phenomenon of increasing validation loss at epoch boundary, remediated via UWSR.}
  \label{fig:one-epoch-three-datasets}
\end{figure}

While dense neural networks may overfit when trained for multiple epochs,
recommendation models often exhibit a peculiar failure at the start of the second epoch: out-of-sample performance deteriorates abruptly while training loss improves.
\citet{zhang2022oneepoch} termed this failure the \emph{one-epoch phenomenon}, hypothesizing that first-seen and revisited training examples induce fundamentally different joint distributions of embeddings and labels. They attribute this discrepancy to sparse features and the rapid adaptation of the shared consumer.
Figure~\ref{fig:one-epoch-three-datasets} demonstrates this phenomenon on three benchmarks.
To address this, mitigations typically intervene directly on the sparse parameters. Common strategies include resetting embedding tables between epochs \citep{meda2024} or attenuating row magnitudes  according to their occurrence frequencies \citep{adamar}.

In this paper, we analyze the one-epoch phenomenon through the lens of the \emph{prequential principle}. During the first epoch, training is \emph{prequential}: the forward pass scores each example before its label can influence the model parameters. During subsequent epochs, or replay, the model consumes embedding features whose state is already conditioned on the very target being predicted.
This causes an implicit augmentation in the loss function at the epoch boundary where the consumer relies on the stored information in the embeddings during replay. We call the resulting train-time change in the effective objective and optima \emph{self-influence asymmetry}.
We quantify this implicit change in the objective function and optima for an exactly solvable scalar model and a local neural approximation, showing the exacerbation of the effect with the reduction in the precision of the embeddings, explaining previous empirical observations of the one-epoch phenomenon being more severe in datasets whose embeddings have low support \citep{zhang2022oneepoch}.

Pathologies caused by breaking prequential training have long been recognized.
The generated-regressor problem
\citep{pagan1984generated} studies downstream regressions whose covariates are
themselves estimated.

For example, suppose a
student's exam score \(y\) is predicted from their mean test score \(x\), but
that mean includes the same exam score \(y\). The fitted relationship will
appear stronger than it is for a new exam excluded from the construction of
\(x\).
In the limiting case in which the mean contains only that exam,
\(x=y\). \citet{pagan1984generated} further shows that when this reuse cannot
be avoided, the estimation uncertainty in the generated regressor must itself
be propagated through the downstream fit, or inference on it will be
overconfident.
These results point to two complementary remediations: exclude
an outcome from the fitted inputs used to predict it, or, when exclusion is
impractical, account for the uncertainty in those inputs.
Modern supervised learning favors exclusion, using techniques like out-of-fold target encoding \citep{micci2001preprocessing}, ordered encodings \citep{prokhorenkova2018catboost}, and cross-fitting \citep{chernozhukov2018double} to ensure no observation constructs its own features. However, because deep recommendation models jointly optimize embeddings and downstream consumers, multi-epoch training is antithetical to this principle.

Our exact  scalar model and local neural analysis show that the resulting self-influence asymmetry is governed by the prediction's sensitivity to each fitted embedding row, and weighted by its uncertainty. This identifies a tractable quantity to control without interrupting end-to-end embedding training.
With these insights, we propose uncertainty-weighted sensitivity regularization (UWSR), a method that penalizes the downstream network's sensitivity to poorly determined sparse rows, making it explicitly costly to depend on  them.
UWSR computes this gradient penalty via standard double backpropagation \citep{drucker1992double}, acting as a functional regularizer on the consumer instead of the embeddings.
Thus, unlike direct sparse-parameter interventions, it counteracts the
network's incentive to exploit the previous influence of the label on the embeddings,  while leaving
them free to be optimized otherwise.

\paragraph{Contributions.}
In summary, our key contributions are:

\begin{enumerate}
    \item We show that \emph{self-influence asymmetry}, an effect echoing the
    classical generated-regressor problem, contributes to the one-epoch
    phenomenon. Using an exactly solvable scalar model and local rowwise analysis on deep recommendation models reveals a mismatch between the consumer a
    training example rewards and the consumer a fresh example rewards. We use an embedding-consumer-update-order intervention to verify this hypothesis.
    \item We derive UWSR from \emph{self-influence asymmetry}, which counteracts the mismatch and provide a computationally tractable method; and
    \item We empirically show that UWSR enables multi-epoch training across
    three datasets (MovieLens-20M, Amazon Electronics, and Avazu), consistently
    outperforming both single-epoch checkpoints and multi-epoch baselines from
    state-of-the-art methods.
\end{enumerate}

\section{Related Work}

\paragraph{The one-epoch phenomenon.}
The sharp degradation of deep recommendation models during the second epoch \citep{zhang2022oneepoch} has previously been attributed to distributional shifts in sparse embeddings and rapid network adaptation. However, this perspective conflates general embedding maturity with the leakage of the currently scored label. Prior experiments find that standard regularizers such as $\ell_2$ decay and embedding dropout do not remove the failure \citep{zhang2022oneepoch}; proposed remedies instead reset embedding tables \citep{meda2024} or attenuate their magnitudes \citep{adamar}.

\paragraph{Target exclusion and fitted representations.}
Established statistical frameworks recognize that target encoding leaks information if a representation includes the instance's own label \citep{micci2001preprocessing}. To prevent this, they explicitly enforce target exclusion using techniques such as ordered statistics \citep{prokhorenkova2018catboost} or cross-fitting \citep{chernozhukov2018double}.
Generated-regressor theory studies the downstream consequences of plugging in an estimated quantity \citep{pagan1984generated}. Although these fields motivate target exclusion and uncertainty propagation, they do not address the deep-learning setting in which intermediate representations (sparse embeddings) and their downstream consumers are jointly and repeatedly optimized.
\paragraph{Prequential prediction, influence, and leverage.}
Our theoretical framework builds on the prequential principle, which dictates that a model must generate a prediction before observing its target \citep{dawid1984prequential}. To quantify the exact mathematical consequences when this temporal separation breaks down, we draw upon influence functions, which decompose a fitted parameter into leave-one-out and self-influence components \citep{koh2017influence}. This self-influence is closely tied to statistical leverage and its extension to classification via generalized leverage \citep{pregibon1981logistic}, which formalizes how strongly an individual observation shifts its own fitted value.
Efron's covariance-penalty theory expresses in-sample optimism through the
covariance between a response and its fitted prediction
\citep{efron2004estimation}.

\paragraph{Input sensitivity and information-weighted parameter methods.}
The theoretical equivalence between additive input noise and derivative penalties is a classical result \citep{bishop1995training}, laying the groundwork for techniques such as double backpropagation and input-Jacobian regularization, which explicitly penalize a model's sensitivity to external inputs \citep{drucker1992double}. A parallel literature addresses parameter uncertainty: methods such as AdaGrad, K-FAC, and elastic weight consolidation use expected curvature or Fisher information to selectively scale parameter updates and penalize displacement \citep{duchi2011adagrad, martens2015kfac, kirkpatrick2017ewc}, while Bayesian neural networks achieve a similar goal through uncertainty sampling \citep{blundell2015weight}.
Our proposed UWSR bridges these approaches, applying input-sensitivity penalties not to external data, but to the model's own fitted, information-weighted representations.

\section{Self-Influence Asymmetry} \label{sec:analysis}
In this section, we demonstrate the mechanics of self-influence asymmetry and its dependence on the support of an embedding row. We first use an exactly solvable scalar model to exactly quantify this asymmetry and then generalize the concept to neural networks. Finally, we demonstrate the impact of this phenomenon using an embedding-consumer-update-order intervention experiment that deliberately breaks prequential training within the first epoch.

\subsection{Scalar Model of the Leave-In/Leave-Out Asymmetry}
\label{sec:scalar}

In this section, we use a scalar model to quantify self-influence asymmetry and how replay implicitly distorts the prequential training objective function, exacerbated as the support for an embedding reduces. We then propose a method to resolve this asymmetry by augmenting the loss function.

Consider an ordered dataset of interactions $\mathcal{D}=\{(x_t,y_t)\}_{t=1}^T$, where $x_t$ maps to a single scalar embedding $e_x$ in the embedding table $E$ and $y_t$ is the continuous target estimated as $\hat{y}_t = w e_{x_t}$, where $w$ represents the shared consumer.
For an entity $x$ that occurs $N_x$ times in $\mathcal{D}$, assume
$y_k = \mu_x + \epsilon_k$ for $k \in \{1, \dots, N_x\}$,
where $\mu_x$ denotes the latent signal associated with $x$, and $\epsilon_k$ denotes the noise specific to occurrence $k$, independent of $\mu_x$.
Assume both are zero-mean, with $\operatorname{Var}(\mu_x)=\tau^2$ and $\operatorname{Var}(\epsilon_k)=\nu^2$ and the noise instances $\{\epsilon_k\}$ are i.i.d.

To quantify self-influence asymmetry, we use a two-stage setup: we first fit embedding $e_x$ on observations of $x$ in $\mathcal{D}$, then hold it fixed to derive the optimal consumer $w$ across in-sample and fresh targets.
We fit the embedding $e_x$ using the empirical squared loss for entity $x$:
\[
L_{\mathrm{fit},x}(e)
=
\frac{1}{2N_x}\sum_{k=1}^{N_x}(y_k-e)^2,
\qquad
\hat{e}_x = \arg\min_{e\in\mathbb{R}}L_{\mathrm{fit},x}(e) = \frac{1}{N_x}\sum_{k=1}^{N_x} y_k = \mu_x+\bar{\epsilon}_x,
\]
where
$\bar{\epsilon}_x=N_x^{-1}\sum_{k\in\mathcal{K}_x}\epsilon_k$.
As $\epsilon_k$ is independent of $\mu_x$, $\mathbb{E}[\hat{e}_x\mid\mu_x]=\mu_x$, making $\hat{e}_x$ a noisy estimate of $\mu_x$, with estimation error $\hat{e}_x-\mu_x=\bar{\epsilon}_x$ and variance $\Sigma_x = \operatorname{Var}(\hat{e}_x-\mu_x) = \operatorname{Var}(\bar{\epsilon}_x) = \tfrac{\nu^2}{N_x}$.

Having characterized the fitted embedding, we now turn to study how strongly the consumer should rely on it, where $w=0$ ignores the embedding and $w=1$ uses it without attenuation.
Holding $\hat{e}_x$ fixed, we evaluate the reliance of consumer $w$ by comparing two targets scored using it: a fresh target $y_{\text{fresh}} = \mu_x + \epsilon_{\text{fresh}}$ ($\epsilon_{\text{fresh}}$ is i.i.d. with $\{\epsilon_k\}$) and a replayed target $y_{\text{replay}} = \mu_x + \epsilon_{\text{replay}}$ (chosen from $\mathcal{D}$, used to fit $\hat{e}_x$). We fit $w$ using the loss functions:
\[
L_{\mathrm{fresh},x}(w) =
\mathbb{E}\!\left[(y_{\mathrm{fresh}}-w\hat{e}_x)^2\right],
\qquad
L_{\mathrm{replay},x}(w) =
\mathbb{E}\!\left[(y_{\text{replay}}-w\hat{e}_x)^2\right].
\]

The two losses differ solely through the target and fitted embedding covariance.
For fresh targets, $\epsilon_{\mathrm{fresh}}$ is independent of $\hat{e}_x$, thus
\(\operatorname{Cov}(\hat{e}_x,y_{\mathrm{fresh}})
=
\operatorname{Cov}
\left(
\mu_x+\bar{\epsilon}_x,
\mu_x+\epsilon_{\mathrm{fresh}}
\right)
=
\operatorname{Var}(\mu_x)
=
\tau^2
\)
and,
\[
L_{\mathrm{fresh},x}(w)
=
(1-w)^2\tau^2+\nu^2+w^2\Sigma_x.
\]
For replay, $y_{\text{replay}}$ contributed to $\hat{e}_x$, hence
$\operatorname{Cov}(\hat{e}_x,y_{\text{replay}})
=
\tau^2+\operatorname{Cov}(\bar{\epsilon}_x,\epsilon_{\text{replay}})
=
\tau^2+\Sigma_x$ and,
\[
L_{\mathrm{replay},x}(w)
=
L_{\mathrm{fresh},x}(w)-2w\Sigma_x.
\]
Thus, for $w\geq 0$, replay obtains an additional loss reduction of
$2w\Sigma_x$ that is absent during fresh evaluation.
Solving first-order optimality conditions for $w$ for the above two loss functions yields,
\[
\boxed{w_{\mathrm{fresh},x}^{\star}
=
\frac{\tau^2}{\tau^2+\Sigma_x},
\qquad
w_{\mathrm{replay},x}^{\star}
=
1.}
\]
Thus, fresh evaluation favors shrinking the contribution of an uncertain embedding, while replay favors using the embedding without attenuation. If the consumer adopts the replay-optimal value
$w_{\mathrm{replay},x}^{\star}=1$, its excess loss on a fresh target is
\[
L_{\mathrm{fresh},x}(w_{\mathrm{replay},x}^{\star})
-
L_{\mathrm{fresh},x}(w_{\mathrm{fresh},x}^{\star})
=
\frac{\Sigma_x^2}{\tau^2+\Sigma_x}.
\]
Since $\Sigma_x=\nu^2/N_x$, both the difference between the two optimal
consumers and the resulting excess fresh loss grow with a decrease in the embedding support $N_x$.

\subsection{Generalizing to Neural Networks} \label{sec:nueral_analysis}

We now demonstrate self-influence asymmetry in neural recommendation systems. Consider an ordered dataset $\mathcal{D}=\{(x_t,y_t)\}_{t=1}^T$ with targets $y_t\in\{0,1\}$. The model estimates $\hat{y}_t$, the probability $y_t=1$, as $\hat{y}_t = \sigma(z_t) = (1 + \exp(-z_t))^{-1}$, governed by two parameterized components:
\begin{itemize}
    \item \textbf{Sparse Embedding Tables $\{E^{(F)}\}_{F=1}^M$}:
    A collection of $M$ embedding tables, with table $ E^{(F)}\in\mathbb{R}^{n_F\times d_F}$ containing $n_F$ rows of dimension $d_F$. For an input $x_t$, $\mathcal{A}(x_t)$ denotes the set of active table-row pairs and $E_{x_t}:=\{e_{F,i}:(F,i)\in\mathcal{A}(x_t)\}$ the retrieved embeddings.
    \item \textbf{Shared Consumer Network ($\w$)}: Combines $E_{x_t}$ and remaining continuous features to output a scalar logit $z_t = z(E_{x_t}, \w)$.
\end{itemize}
The model uses the binary cross-entropy (BCE) loss: $\ell(z_t,y_t) = -y_t\log\hat{y}_t-(1-y_t)\log(1-\hat{y}_t)$.
Following the two-stage abstraction for the scalar model, we first isolate the fitting of an embedding from the downstream consumer using a local row-wise analysis. For a fixed occurrence $t$ and active pair $(F,i) \in \mathcal{A}(x_t)$, we hold the remaining embeddings and consumer fixed.
Let $\mathcal{G}_{t^{-}}$ denote the pre-update information set containing $x_t$ and prior interactions, excluding $y_t$ and updates from $(x_t,y_t)$. The pre-update logit $z_t$ and prediction $\hat{y}_t$ are deterministic given $\mathcal{G}_{t^{-}}$. We define the conditionally optimal prediction $\hat{y}_t^\star = \Pr(y_t=1 \mid \mathcal{G}_{t^{-}})$, its logit $z_t^\star = \sigma^{-1}(\hat{y}_t^\star)$, and logit error $\Delta z_t = z_t - z_t^\star$.
To evaluate prediction costs, we introduce three core quantities to characterize the active embedding row $(F,i)$ and the network's reliance on it:
\begin{itemize}
    \item \textbf{Fitted Row \& Error:} Let $\hat{e}_{F,i}^-$ denote the fitted value of the active row in $\mathcal{G}_{t^-}$, and let $e_{F,i}^*$ denote the reference value yielding target logit $z_t^\star$. Thus, the pre-update row error is defined as $\Delta e_{F,i}^- = \hat{e}_{F,i}^- - e_{F,i}^*$.
    \item \textbf{Embedding Uncertainty:} We model row parameter uncertainty due to finite observations via a Laplace approximation with pre-update precision $\mathcal{I}_{F,i}^- > 0$ (the damped row-wise generalized Gauss-Newton information\footnote{The row-wise GGN forms the curvature block for $e_{F,i}$, scaling logit gradient outer products by BCE curvature. Damping guarantees invertibility, making $(\mathcal{I}_{F,i}^-)^{-1}$ a valid posterior covariance for rarely observed rows.}).
    We denote the expectation over this posterior as $\mathbb{E}_*$, thus the error covariance satisfies $\mathbb{E}_*[\Delta e_{F,i}^- (\Delta e_{F,i}^-)^\top] = (\mathcal{I}_{F,i}^-)^{-1}$.
    \item \textbf{Local Logit Sensitivity:} Extending beyond linear models where sensitivity is constant, we capture the neural consumer's local reliance on $e_{F,i}$ via the logit sensitivity vector $s_{F,i,t} = \nabla_{e_{F,i}} z(E_{x_t}, \w)\vert_{e_{F,i}=e_{F,i}^*}$. A first-order Taylor expansion yields $\Delta z_t \approx s_{F,i,t}^\top \Delta e_{F,i}^-$.
\end{itemize}
With this setup, we first characterize the cost of this embedding uncertainty during a fresh evaluation, and then show how replaying this occurrence exploits this embedding uncertainty.

\subsubsection{Implicit Penalization of Uncertain Embeddings in Prequential Training}
\label{sec:first-pass-exclusion}

This section demonstrates how prequential training naturally penalizes reliance on uncertain embedding rows, by implicitly introducing an excess loss penalty proportional to embedding uncertainty, thus showing that first-epoch training naturally penalizes the downstream network for relying on poorly determined representations.

For a fixed example $t$, consider the Taylor expansion of $\ell(z_t, y_t)$ around $z_t^*$:
\begin{equation} \label{eq:taylor_3.2.1_init}
\ell(z_t, y_t) - \ell(z_t^*, y_t) = (\hat{y}_t^* - y_t)\Delta z_t + \frac{1}{2}\hat{y}_t^*(1 - \hat{y}_t^*)(\Delta z_t)^2 + \mathcal{O}(\vert{}\Delta z_t\vert{}^3)
\end{equation}
As $\Delta z_t$ is deterministic given $\mathcal{G}_{t^-}$, and $y_t$ has conditional expectation $\mathbb{E}[y_t \mid \mathcal{G}_{t^-}] = \hat{y}_t^*$, the first-order interaction in \eqref{eq:taylor_3.2.1_init} gives
$\mathbb{E}[(\hat{y}_t^* - y_t)\Delta z_t \mid \mathcal{G}_{t^-}] = \Delta z_t (\hat{y}_t^* - \mathbb{E}[y_t \mid \mathcal{G}_{t^-}]) = 0$.
Taking the conditional expectation of \eqref{eq:taylor_3.2.1_init} while removing higher-order terms yields,
\begin{equation} \label{eq:taylor_conditional_3.2.1}
    \mathbb{E}[\ell(z_t, y_t) - \ell(z_t^*, y_t) \mid \mathcal{G}_{t^-}] \approx \frac{1}{2}\hat{y}_t^*(1 - \hat{y}_t^*)(\Delta z_t)^2 \ge 0
\end{equation}
Substituting $\Delta z_t \approx s_{F,i,t}^\top \Delta e_{F,i}^-$ and taking expectation $\mathbb{E}_*$ over the posterior row uncertainty
yields:
\begin{equation} \label{eq:delta_z_e_start}
    \mathbb{E}_*[(\Delta z_t)^2] \approx s_{F,i,t}^\top \mathbb{E}_*[\Delta e_{F,i}^- (\Delta e_{F,i}^-)^\top] s_{F,i,t} = s_{F,i,t}^\top (\mathcal{I}_{F,i}^-)^{-1} s_{F,i,t}
\end{equation}
Combining \eqref{eq:taylor_conditional_3.2.1} and \eqref{eq:delta_z_e_start} gives the expected local excess risk under first-epoch prequential training:
\begin{equation} \label{eq:first-pass-uncertainty-cost}
\boxed{
\mathbb{E}_\star\left[\mathbb{E}\!\left[
    \ell(z_t,y_t)-\ell(z_t^\star,y_t)
    \,\middle|\,
    \mathcal{G}_{t^{-}}
\right]\right]
\approx
\frac{1}{2}
\hat{y}_t^\star(1-\hat{y}_t^\star)
s_{F,i,t}^{\top}
(\mathcal{I}_{F,i}^{-})^{-1}
s_{F,i,t}
\geq 0.
}
\end{equation}
This local excess demonstrates that prequential training naturally exposes a positive quadratic penalty on downstream logit sensitivity. Predictions incur the highest excess loss when the shared consumer network exhibits high sensitivity $s_{F,i,t}$ toward poorly informed embedding directions, that is, rows where $\mathcal{I}_{F,i}^-$ is small.

\subsubsection{Replay-Induced Exploitation of Uncertain Embeddings}
\label{sec:self-influence}

In this subsection, we show how replay encourages the consumer to rely on uncertain embeddings and cause the self-influence asymmetry.
Let $r_t = y_t - \hat{y}_t$ denote the prediction residual, where $\hat{y}_t$ is deterministic given $\mathcal{G}_{t^-}$. When occurrence $t$ updates the model, it writes a label-aligned component into the active row $e_{F,i}$. We define this update as the self-induced row displacement, denoted by $\delta e_{F,i,t}^{\text{self}}$. When occurrence $t$ is subsequently replayed, the shared consumer scores this modified embedding row, shifting the output logit.
We define this output shift as the logit displacement, $u_t$.

To measure the change in loss post replay, let $y_t'$ denote an independent target label for $x_t$ which did not contribute to updating the embedding rows. Comparing the loss on the observed label $y_t$ against the expected loss on an independent target $y_t'$ defines the leave-in optimism as,
$$D_{F,i,t} = \ell(z_t + u_t, y_t) - \ell(z_t, y_t) - \mathbb{E}[\ell(z_t + u_t, y_t') - \ell(z_t, y_t') \mid \mathcal{G}_{t^-}, x_t, y_t].$$
We formalize the relationship between these quantities in the proposition to follow.

\begin{proposition}
\label{prop:leave_in_optimism}
In the local Generalized Gauss-Newton (GGN) row model with post-update row information $\mathcal{I}_{F,i} = \mathcal{I}_{F,i}^- + h_t s_{F,i,t} s_{F,i,t}^\top$ and variance scale $h_t = \hat{y}_t(1 - \hat{y}_t)$ evaluated at the model's own pre-update prediction, the following properties hold:
\begin{enumerate}
    \item The self-induced row displacement $\delta e_{F,i,t}^{\text{self}}$ and resulting logit shift $u_t$ satisfy
    \begin{equation}
        \delta e_{F,i,t}^{\text{self}} = r_t \mathcal{I}_{F,i}^{-1} s_{F,i,t}\quad \text{and} \quad u_t = r_t \|s_{F,i,t}\|_{\mathcal{I}_{F,i}^{-1}}^2
    \end{equation}

    \item For an independent target $y_t'\sim\mathrm{Bernoulli}(\hat y_t)$, the leave-in optimism $D_{F,i,t}$ and its conditional expectation satisfy
    \begin{equation} \label{eq:expected-leave-in-optimism}
        D_{F,i,t} = -r_t^2 \|s_{F,i,t}\|_{\mathcal{I}_{F,i}^{-1}}^2 \quad \text{and} \quad -\mathbb{E}[D_{F,i,t} \mid \mathcal{G}_{t^-}] = h_t \|s_{F,i,t}\|_{\mathcal{I}_{F,i}^{-1}}^2 \in [0, 1).
    \end{equation}
\end{enumerate}
\end{proposition}
Appendix~\ref{app:alignment} proves Proposition~\ref{prop:leave_in_optimism}.
Comparing \eqref{eq:first-pass-uncertainty-cost} and \eqref{eq:expected-leave-in-optimism} reveals the opposing incentives of prequential training and replay. Target exclusion penalizes logit sensitivity under $\mathcal{I}_{F,i,-t}^{-1}$, whereas self-inclusion rewards it under $\mathcal{I}_{F,i}^{-1}$.
As rare or low-support embedding rows accumulate minimal information $\mathcal{I}_{F,i}$, their inverse norm $\Vert{}\cdot\Vert{}_{\mathcal{I}_{F,i}^{-1}}^2$ is largest, yielding the highest generalized leverage during replay. As multi-epoch training actively breaks prequential training, it actively incentivizes the shared consumer network to inflate its score sensitivity $s_{F,i,t}$ toward poorly determined representations, exploiting target leakage to suppress training loss at the expense of generalization. UWSR is designed to restore the missing positive pressure during later epochs.

\subsection{Embedding-Consumer-Update-Order Intervention Experiment}
\label{sec:update-order}

In this section, we present empirical evidence that breaking same-label self-inclusion drives sensitivity inflation by inducing these conditions in the first epoch itself.
In ordinary training, both embeddings and consumer are updated within a single forward-backward pass.
We alter this process by performing two passes for each minibatch while matching the total updates of standard joint training.
In embedding$\rightarrow$consumer, only embeddings are updated in the first pass and the examples are rescored to update the consumer. In consumer$\rightarrow$embedding, this order is reversed\footnote{The experimental setup employed here is representative of modern recommendation systems, detailed in Section~\ref{sec:results}.}.

Intuitively, embedding-first writes a small, label-aligned trace into each active row before the consumer is trained. Rescoring then makes those examples easier to predict, encouraging the consumer to become more sensitive to the modified rows. This lowers training loss but provides no advantage on fresh examples. Consumer-first prevents the consumer from exploiting this trace, so we expect much smaller increases in sensitivity and validation loss.
Table~\ref{tab:update-order} reports results after a single epoch. While consumer-first remains close to ordinary training, embedding-first creates a large train loss advantage but substantially increases score sensitivity and worsens validation performance. Figure~\ref{fig:support-sensitivity} further shows this inflation is strongest for poorly supported rows. Together, these results demonstrate that same-label self-inclusion is sufficient to induce sensitivity inflation within the first epoch.

\begin{table}[t]
\centering
\footnotesize
\caption{One-epoch embedding-consumer-update-order intervention, averaged over two seeds.
$\Delta\mathrm{BCE}_{\mathrm{train}}$ is the reduction in same-minibatch
training loss after the first block update.
$\Delta\mathrm{BCE}_{\mathrm{val}}$ is the validation-loss reduction relative
to ordinary training; positive values indicate improvement.
$\Delta S$ is the relative change in primary-row logit sensitivity relative
to joint training.}
\label{tab:update-order}
\setlength{\tabcolsep}{5pt}
\begin{tabular}{@{}lrrrrrr@{}}
\toprule
& \multicolumn{3}{c}{Embedding$\rightarrow$Consumer}
& \multicolumn{3}{c}{Consumer$\rightarrow$Embedding} \\
\cmidrule(lr){2-4}\cmidrule(lr){5-7}
Dataset
& $\Delta\mathrm{BCE}_{\mathrm{train}}$
& $\Delta\mathrm{BCE}_{\mathrm{val}}$
& $\Delta S$
& $\Delta\mathrm{BCE}_{\mathrm{train}}$
& $\Delta\mathrm{BCE}_{\mathrm{val}}$
& $\Delta S$ \\
\midrule
MovieLens-20M
& $+1.43\%$ & $-2.90\%$ & $+54.0\%$
& $+0.19\%$ & $-0.21\%$ & $-0.1\%$ \\
Amazon Electronics
& $+10.93\%$ & $-13.14\%$ & $+474.3\%$
& $+0.22\%$ & $-0.44\%$ & $-4.8\%$ \\
Avazu
& $+1.73\%$ & $-1.86\%$ & $+56.4\%$
& $+0.15\%$ & $+0.31\%$ & $+4.6\%$ \\
\bottomrule
\end{tabular}
\end{table}

\begin{figure}[t]
  \centering
  \includegraphics[width=\textwidth]{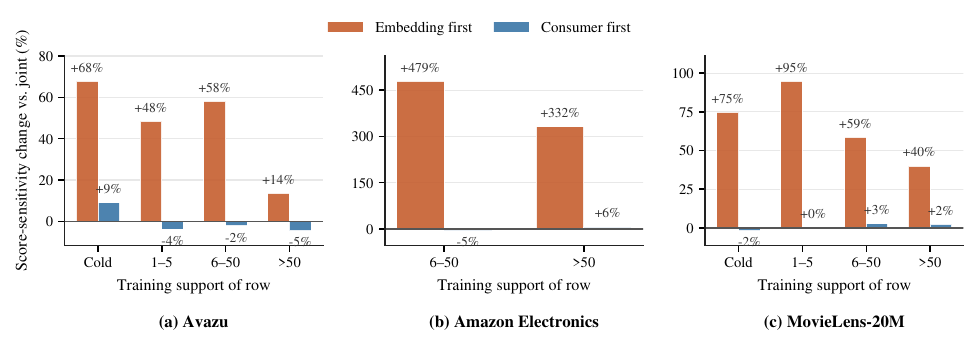}
  \vspace{-2em}
  \caption{
  Mean percentage change in primary-row logit sensitivity for embedding-consumer-update-order interventions compared to ordinary joint training for Avazu, Amazon Electronics and MovieLens-20M datasets grouped by row support.
  }
  \label{fig:support-sensitivity}
\end{figure}

\section{Uncertainty-Weighted Sensitivity Regularization}
In this section, we propose uncertainty-weighted sensitivity regularization (UWSR), which counteracts self-influence asymmetry by adding a functional regularizer on the dense consumer network, restoring the positive inverse-information penalty present during the first epoch.

The local self-inclusion advantage is identified as \(-r_t^2s_{F, i, t}^{\mathsf T}\mathcal I_{F, i}^{-1}s_{F, i, t}\) in \eqref{eq:expected-leave-in-optimism}.
As \(|r_t|\leq1\), its magnitude is bounded by \(s_{F, i, t}^{\mathsf T}\mathcal I_{F, i}^{-1}s_{F, i, t}\).
With \(s_{F,i}(x)=\nabla_{e_{F,i}}z(x)\), UWSR restores this positive pressure by adding the following penalty to the loss after the first epoch:
\begin{align}
    R(x)=\sum_{(F,i)\in\mathcal A(x)}
    s_{F,i}(x)^{\mathsf T}\widehat\Sigma_{F,i}s_{F,i}(x),
    \qquad
    \mathcal L_{\mathrm{UWSR}}(x,y)=\ell(z(x),y)+\lambda R(x),
  \label{eq:uwsr-objective}
\end{align}
where  \(\widehat\Sigma_{F,i}\) are the frozen inverse-information weights.
This is the same inverse-information geometry maintained by prequential training, as identified in \eqref{eq:first-pass-uncertainty-cost}.
To deploy UWSR in practice, we estimate \(\widehat\Sigma_{F,i}\) during epoch one, keep it fixed and apply the sensitivity penalty efficiently in subsequent epochs.

\begin{itemize}
\item \textbf{Estimation of \(\widehat\Sigma_{F,i}\)}.
We estimate \(\widehat\Sigma_{F,i}\) using the squared-gradient accumulator \citep{squisher} that adaptive optimizers
such as Adam already maintain, repurposed as an approximation of the
diagonal Fisher information. Rather than one scalar per parameter, we
accumulate one scalar per embedding row, the same row-wise reduction used by
row-wise Adagrad and row-wise Adam; Appendix~\ref{sec:construction} gives the
exact construction and its per-coordinate ablation. This is not
identical to the Fisher (\citet{squisher} characterize the exact differences
between the two), but across five applications it performs comparably to the
Fisher while requiring no additional computation, since the accumulator is
already computed as part of ordinary optimizer state. We adopt it as our
default estimator for this reason; the expected-Fisher and uniform-weight
variants are compared directly in Appendix~\ref{app:alternatives}.
\item \textbf{Computing \(s_{F,i}(x)\)}:
    This can be computed efficiently via double backpropagation \citep{drucker1992double} and second-order automatic differentiation, available in PyTorch via \texttt{create\_graph=True}.
\end{itemize}

\section{Results} \label{sec:results}
\paragraph{Evaluation Protocol}
We evaluate on MovieLens-20M \citep{harper2015movielens}, AmazonElectronics \citep{McAuley,zhou2018din}, and Avazu \citep{avazu-ctr-prediction} using training, validation, and test partitions. Validation selects all hyperparameters; a disjoint subset of the validation data is used only for the affine-calibration diagnostic. The test partition is not used for selection.
We report raw BCE and AUC, with AUC as the primary metric;
Appendix~\ref{app:data} gives dataset details.
The model's shared consumer is a compact WuKong-inspired interaction model \citep{zhang2024wukong} and  methods are trained for four epochs unless noted.
Training examples are independently reshuffled on every pass, using the same pass-specific permutation for every arm; the epoch-boundary effect therefore does not require replaying the same minibatch order. Appendix~\ref{app:init} provides further details.

\paragraph{Comparisons}
We compare one epoch, naive four-epoch replay, MEDA \citep{meda2024}, AdamAR \citep{adamar}, active-row AdamW, embedding dropout \citep{srivastava2014dropout}, and embeddings frozen until the final epoch. All methods other than AdamAR and active-row AdamW use the Adam \citep{kingma2015adam} optimizer.
Alongside UWSR, we evaluate UWSR Expected Fisher and UWSR Uniform, described in Appendix~\ref{app:alternatives}.
As multi-epoch arms use four times the gradient steps of the one-epoch baseline, the comparison asks whether additional training is useful.

\paragraph{Test-Set Performance} Table~\ref{tab:test} summarizes the empirical results, showing that UWSR shows consistent improvement across all datasets beyond one epoch against naive replay, the one-epoch checkpoint, and state-of-the-art methods. Further results on more architectures are summarized in Appendix~\ref{app:architecture-results}. The method also incurs minimal computational overhead (wall-clock cost in Appendix~\ref{app:compute}).

\paragraph{Training diagnosis} Figure~\ref{fig:batch-diagnostics} shows the impact of using UWSR on the validation loss, the logit sensitivity and active embedding norms against naive replay and AdamAR. While UWSR is able to maintain a low logit sensitivity and improvements in validation loss beyond one epoch, it does not place any direct restraints on the embeddings, allowing free growth of the embedding norms as shown.

\definecolor{uwsrgood}{HTML}{2A8F6A}
\definecolor{uwsrbad}{HTML}{C85B61}
\providecommand{\poscell}[2]{\cellcolor{uwsrgood!#1}#2}
\providecommand{\negcell}[2]{\cellcolor{uwsrbad!#1}#2}
\begin{table}[t]
\centering
\scriptsize
\caption{Test-set performance after each epoch, averaged over five seeds. Each endpoint uses the coefficient selected on validation at that endpoint. The right block gives the paired relative BCE reduction against the same seed's one-epoch checkpoint; positive values are better.
}
\label{tab:test}
\setlength{\tabcolsep}{3pt}
\renewcommand{\arraystretch}{0.94}
\begin{tabular}{@{}llrrrrrrr@{}}
\toprule
& & \multicolumn{4}{c}{Test AUC} & \multicolumn{3}{c}{$\Delta$BCE (\%)} \\
\cmidrule(lr){3-6}\cmidrule(lr){7-9}
Dataset & Method & E1 & E2 & E3 & E4 & E2 & E3 & E4 \\
\midrule
Amazon & Naive Replay & 0.8523 & \poscell{14}{0.8594} & \negcell{11}{0.8403} & \negcell{16}{0.8233} & \poscell{5}{$+0.02$} & \negcell{14}{$-28.51$} & \negcell{26}{$-95.43$} \\
 & MEDA & 0.8523 & \poscell{15}{0.8608} & \poscell{18}{0.8644} & \poscell{19}{0.8649} & \poscell{15}{$+2.54$} & \poscell{18}{$+3.62$} & \poscell{19}{$+3.92$} \\
 & Freeze until final epoch & 0.8523 & \negcell{8}{0.8469} & \negcell{7}{0.8483} & \negcell{9}{0.8460} & \negcell{6}{$-1.61$} & \negcell{5}{$-1.10$} & \negcell{6}{$-1.74$} \\
 & Active-row AdamW & 0.8523 & \poscell{14}{0.8594} & \negcell{25}{0.7852} & \negcell{26}{0.7811} & \poscell{5}{$+0.02$} & \negcell{12}{$-20.06$} & \negcell{13}{$-23.54$} \\
 & Embedding dropout & 0.8523 & \poscell{12}{0.8575} & \negcell{9}{0.8453} & \negcell{10}{0.8417} & \poscell{9}{$+0.72$} & \negcell{7}{$-4.03$} & \negcell{8}{$-7.16$} \\
 & AdamAR & 0.8523 & \poscell{18}{0.8640} & \poscell{20}{0.8660} & \poscell{19}{0.8656} & \poscell{18}{$+3.40$} & \poscell{17}{$+3.07$} & \poscell{17}{$+2.92$} \\
 & UWSR Uniform & 0.8523 & \poscell{14}{0.8591} & \poscell{10}{0.8556} & \negcell{5}{0.8516} & \poscell{13}{$+1.79$} & \negcell{5}{$-0.33$} & \negcell{5}{$-1.21$} \\
 & UWSR & 0.8523 & \poscell{22}{\textbf{0.8687}} & \poscell{25}{\textbf{0.8735}} & \poscell{26}{0.8754} & \poscell{22}{$+5.10$} & \poscell{25}{\textbf{$+6.43$}} & \poscell{26}{$+6.78$} \\
 & UWSR Expected Fisher & 0.8523 & \poscell{22}{\textbf{0.8687}} & \poscell{25}{0.8734} & \poscell{26}{\textbf{0.8756}} & \poscell{22}{\textbf{$+5.13$}} & \poscell{25}{\textbf{$+6.43$}} & \poscell{26}{\textbf{$+6.93$}} \\
\midrule
MovieLens & Naive Replay & 0.6766 & \negcell{12}{0.6436} & \negcell{18}{0.5926} & \negcell{26}{0.5097} & \negcell{7}{$-1.89$} & \negcell{12}{$-7.08$} & \negcell{26}{$-33.66$} \\
 & MEDA & 0.6766 & \poscell{11}{0.6800} & \negcell{6}{0.6711} & \negcell{7}{0.6707} & \poscell{18}{$+0.96$} & \poscell{6}{$+0.03$} & \poscell{9}{$+0.21$} \\
 & Freeze until final epoch & 0.6766 & \negcell{7}{0.6703} & \negcell{11}{0.6502} & \negcell{10}{0.6560} & \poscell{8}{$+0.13$} & \negcell{7}{$-1.39$} & \negcell{6}{$-0.90$} \\
 & Active-row AdamW & \negcell{4}{0.6762} & \poscell{16}{0.6840} & \poscell{14}{0.6821} & \poscell{7}{0.6777} & \poscell{14}{$+0.62$} & \poscell{16}{$+0.77$} & \poscell{7}{$+0.10$} \\
 & Embedding dropout & 0.6766 & \poscell{12}{0.6802} & \negcell{9}{0.6570} & \negcell{13}{0.6360} & \poscell{16}{$+0.81$} & \negcell{6}{$-0.67$} & \negcell{8}{$-2.63$} \\
 & AdamAR & 0.6766 & \negcell{5}{0.6756} & \negcell{8}{0.6661} & \negcell{5}{0.6752} & \poscell{10}{$+0.30$} & \negcell{6}{$-0.69$} & \negcell{6}{$-0.86$} \\
 & UWSR Uniform & 0.6766 & \poscell{18}{0.6858} & \poscell{25}{\textbf{0.6940}} & \poscell{14}{0.6820} & \poscell{17}{$+0.90$} & \poscell{26}{\textbf{$+2.00$}} & \poscell{20}{$+1.19$} \\
 & UWSR & 0.6766 & \poscell{26}{\textbf{0.6951}} & \poscell{21}{0.6892} & \poscell{14}{\textbf{0.6824}} & \poscell{26}{\textbf{$+1.97$}} & \poscell{25}{$+1.84$} & \poscell{21}{\textbf{$+1.38$}} \\
 & UWSR Expected Fisher & 0.6766 & \poscell{26}{0.6947} & \poscell{21}{0.6896} & \poscell{11}{0.6796} & \poscell{26}{$+1.94$} & \poscell{24}{$+1.75$} & \poscell{20}{$+1.18$} \\
\midrule
Avazu & Naive Replay & 0.7405 & \negcell{17}{0.7148} & \negcell{22}{0.6976} & \negcell{26}{0.6837} & \negcell{15}{$-4.51$} & \negcell{21}{$-9.25$} & \negcell{26}{$-13.85$} \\
 & MEDA & 0.7405 & \poscell{6}{0.7409} & \poscell{10}{0.7425} & \poscell{7}{0.7413} & \poscell{6}{$+0.08$} & \poscell{5}{$+0.03$} & \negcell{5}{$-0.11$} \\
 & Freeze until final epoch & 0.7405 & \negcell{6}{0.7388} & \poscell{7}{0.7411} & \negcell{5}{0.7401} & \negcell{6}{$-0.46$} & $+0.00$ & \negcell{5}{$-0.20$} \\
 & Active-row AdamW & \negcell{9}{0.7348} & \negcell{11}{0.7304} & \negcell{15}{0.7223} & \negcell{16}{0.7191} & \negcell{10}{$-1.82$} & \negcell{13}{$-3.79$} & \negcell{16}{$-5.76$} \\
 & Embedding dropout & 0.7405 & \negcell{22}{0.6992} & \negcell{23}{0.6961} & \negcell{23}{0.6947} & \negcell{15}{$-4.86$} & \negcell{17}{$-6.37$} & \negcell{19}{$-7.77$} \\
 & AdamAR & \poscell{19}{\textbf{0.7489}} & \poscell{18}{0.7481} & \poscell{20}{0.7497} & \poscell{20}{0.7501} & \poscell{19}{$+1.35$} & \poscell{20}{$+1.52$} & \poscell{19}{$+1.46$} \\
 & UWSR Uniform & 0.7405 & \poscell{18}{0.7480} & \poscell{16}{0.7470} & \poscell{5}{0.7407} & \poscell{18}{$+1.32$} & \poscell{18}{$+1.25$} & \poscell{11}{$+0.48$} \\
 & UWSR & 0.7405 & \poscell{25}{0.7554} & \poscell{26}{\textbf{0.7561}} & \poscell{25}{0.7547} & \poscell{25}{$+2.46$} & \poscell{25}{\textbf{$+2.43$}} & \poscell{24}{$+2.29$} \\
 & UWSR Expected Fisher & 0.7405 & \poscell{26}{\textbf{0.7561}} & \poscell{24}{0.7537} & \poscell{26}{\textbf{0.7559}} & \poscell{26}{\textbf{$+2.56$}} & \poscell{24}{$+2.21$} & \poscell{26}{\textbf{$+2.49$}} \\
\bottomrule
\end{tabular}
\end{table}

\begin{figure}[h]
\centering
\includegraphics[width=\textwidth]{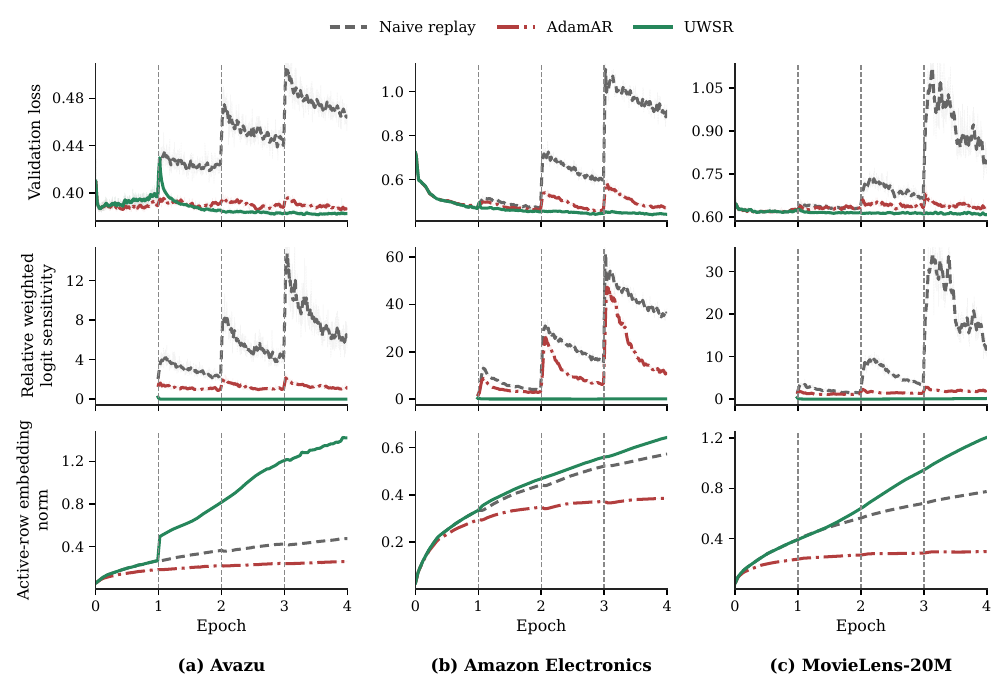}
\caption{
Validation loss and uncertainty-weighted logit
sensitivity ($s_{F,i}(x)^{\mathsf T}\widehat\Sigma_{F,i}s_{F,i}(x)$) measured via empirical Fisher approximation (Appendix~\ref{app:alternatives}) and norm of active embeddings each training step on Avazu, Amazon Electronics, and MovieLens-20M for a single seed. When compared to baseline methods, UWSR is able to control the logit sensitivity, while allowing free growth of embedding norms and show improvements beyond one epoch.
}
\label{fig:batch-diagnostics}
\end{figure}

\section{Conclusion}
Multi-epoch training in sparse-embedding recommendation models suffers from self-influence asymmetry: later passes score examples through rows containing target leakage written by their own earlier updates. Our scalar model, local row analysis, and embedding-consumer update order intervention show that this prequential breakdown creates an incentive for the shared consumer to rely on rarely occurring rows. To resolve this without direct interventions on embeddings, we propose uncertainty-weighted sensitivity regularization (UWSR), which uses double backpropagation weighted by frozen inverse row information to constrain target sensitivity on poorly determined rows.

Our analysis is rowwise and assumes independent embedding updates, meaning it does not model cross-row interactions. Additionally, the formulation considers fixed-window replay, so adapting the method to continual learning would require dynamically updating the row information matrices as new independent examples arrive.

Finally, while our embedding-consumer update order intervention demonstrates that within a minibatch, self-inclusion is sufficient to trigger target sensitivity escalation, further work is needed to measure its exact contribution relative to other factors driving the overall multi-epoch degradation.

\section*{AI Use Disclosure}
In this work, we used generative AI tools to assist with editing and refining the academic writing, as well as to help draft and structure the source code for the algorithms. We did not use generative AI tools to formulate the core scientific concepts, mathematical claims, or experimental designs. The authors have carefully reviewed all AI-assisted content and take full responsibility for it.

\section*{Reproducibility Statement}
To ensure the complete reproducibility of our experiments, all datasets and model architectures utilized in this work are entirely publicly available. Detailed descriptions of the experimental setup, data pre-processing steps, and hyperparameter configurations have been included. The source code implementing our algorithms and training pipelines will be made publicly available upon publication.

\section*{Ethics Statement}
The research presented in this paper focuses on foundational machine learning methodologies. Our experiments rely exclusively on standard, publicly available datasets, and the work does not involve human subjects, personally identifiable information, or sensitive topics. To the best of our knowledge, there are no immediate ethical concerns or direct risks of negative societal impact arising from the specific algorithms and results discussed in this work.

\bibliographystyle{iclr2027_conference}
\bibliography{references}

\clearpage
\appendix
\section{Proof of Proposition 1}
\label{app:alignment}

\begin{proof}[Proof of Proposition~\ref{prop:leave_in_optimism}]

Conditioned on the pre-update reference state $\mathcal G_{t^-}$, the quantities $\mathcal I_{F,i}^-$, $s_{F,i,t}$, and $\hat y_t$ are deterministic. For a displacement $\delta$ to row $e_{F,i}$,  let $\widetilde Q(\delta)$ be a local quadratic approximation for the change in the loss function. Thus,

\begin{align*}
    \widetilde Q(\delta)
    =\tfrac12\delta^{\mathsf T}\mathcal I_{F,i}^-\delta
     -r_ts_{F,i,t}^{\mathsf T}\delta
     +\tfrac12h_t(s_{F,i,t}^{\mathsf T}\delta)^2 .
\end{align*}
The first-order optimality condition on $\widetilde Q(\delta)$, $\mathcal I_{F,i}\delta = r_t s_{F,i,t}$, gives $\delta e_{F,i,t}^{\mathrm{self}}=r_t\mathcal I_{F,i}^{-1}s_{F,i,t}$ and $u_t=r_t\|s_{F,i,t}\|_{\mathcal I_{F,i}^{-1}}^2$, completing the proof of part one.

Defining $a_t=\|s_{F,i,t}\|_{(\mathcal I_{F,i}^-)^{-1}}^2$, the Sherman-Morrison
identity gives
\begin{align*}
    \|s_{F,i,t}\|_{\mathcal I_{F,i}^{-1}}^2=\frac{a_t}{1+h_ta_t}
    \qquad
    \text{and}
    \qquad
    h_t\|s_{F,i,t}\|_{\mathcal I_{F,i}^{-1}}^2=\frac{h_ta_t}{1+h_ta_t}\in[0,1).
\end{align*}

As BCE loss is affine in its label, $\ell(v,y)=\log(1+e^v)-yv$.
Consequently, for any $u$,

\begin{align*}
 \E\left[\ell(z_t+u,y_t) - \ell(z_t+u,y_t')\mid \mathcal{G}_{t^-}, x_t, y_t \right]
  -\E\left[\ell(z_t,y_t)- \ell(z_t,y_t') \mid \mathcal{G}_{t^-}, x_t, y_t\right]
  =-(y_t-\hat y_t)u.
\end{align*}
Substituting $u=u_t$ proves
$D_{F,i,t}=-r_t^2\|s_{F,i,t}\|_{\mathcal I_{F,i}^{-1}}^2$, and the pointwise
bound follows from $|r_t|\leq1$. Under local correct specification
$\hat y_t=\hat y_t^\star$, $\E[r_t^2\mid\mathcal G_{t^-}]=\hat
y_t^\star(1-\hat y_t^\star)=h_t$, which yields the expected identity in part 2.
\end{proof}

\section{UWSR Construction}

\subsection{Method details}
\label{sec:construction}

For sigmoid BCE, the rowwise expected Fisher and generalized Gauss--Newton
information coincide. We instantiate their rowwise scale with an AdaGrad-style
cumulative gradient-energy statistic \citep{duchi2011adagrad}. Let
$g_{F,i}^{B}$ denote the row gradient of the minibatch loss for row $(F,i)$ of
table $F$, whose embeddings have dimension $d_F$:
\begin{align}
    \widehat{\mathcal I}_{F,i}
    &:=\sum_{B:\,i\in\mathcal A(B)}
      \frac{\lVert g_{F,i}^{B}\rVert_2^2}{d_F},
    \qquad
    \widetilde{\mathcal I}_{F,i}
    :=\begin{cases}
      (1-\gamma_{\mathrm{shr}})\widehat{\mathcal I}_{F,i}
        +\gamma_{\mathrm{shr}}\,\overline{\mathcal I}_{F}, & i\in\mathcal O,\\
      \overline{\mathcal I}_{F}, & i\notin\mathcal O,
    \end{cases}
    \label{eq:rowwise-information}\\
    \epsilon_F&:=10^{-6}\overline{\mathcal I}_{F},
    \qquad
    \rho_{F,i}:=\frac{1}{\max\{\widetilde{\mathcal I}_{F,i},\epsilon_F\}},
    \notag\\
    c&:=\frac{1}{|\mathcal O|}\sum_{(F,i)\in\mathcal O}\rho_{F,i},
    \qquad
    \widehat\Sigma_{F,i}:=\frac{\rho_{F,i}}{c}I,
    \label{eq:rowwise-normalization}
\end{align}
where \(\mathcal O\) is the set of rows observed during the first pass,
\(\overline{\mathcal I}_F\) is the mean accumulated information among observed
rows of table \(F\), and \(\gamma_{\mathrm{shr}}\) is a one-time table-level
shrinkage coefficient. The accumulator has no exponential decay or bias
correction: inactive rows simply receive zero additional information. The
resulting weights are frozen after the first pass.

The minibatch construction preserves the Fisher/GGN scale needed here. Let
$g_{F,i,t}$ be occurrence $t$'s row gradient. Writing
$g_{F,i}^{B}=|B|^{-1}\sum_{t\in B_i}g_{F,i,t}$, the usual Fisher score identity
$\E[g_{F,i,t}\mid x_t]=0$ and conditional independence make
cross-occurrence terms vanish, so
\begin{equation}
  \E\!\left[\lVert g_{F,i}^{B}\rVert_2^2\mid B_i\right]
  =|B|^{-2}\sum_{t\in B_i}
    \E\!\left[\lVert g_{F,i,t}\rVert_2^2\mid x_t\right].
  \label{eq:minibatch-information}
\end{equation}
The fixed batch-reduction factor is removed by the normalization in
Equation~\eqref{eq:rowwise-normalization}; accumulation across batches measures
information along the same first-pass trajectory that fits the rows.

The scalarization has a direct optimizer analogue. UWSR uses cumulative
row-wise AdaGrad accumulation in time; the same one-scalar-per-row reduction is
also common in row-wise Adam implementations. UWSR reuses the AdaGrad statistic
in a dual role: the optimizer scales a row update by its inverse square root,
whereas UWSR weights downstream sensitivity by its inverse because it is
propagating row uncertainty:
\[
  \Delta e_{F,i}^{\mathrm{AdaGrad}}
    \propto-\widehat{\mathcal I}_{F,i}^{-1/2}g_{F,i},
  \qquad
  R_{F,i}\propto\big\lVert
    \widehat{\mathcal I}_{F,i}^{-1/2}s_{F,i}
  \big\rVert_2^2.
\]
A per-coordinate UWSR construction gave
essentially identical performance; Appendix~\ref{app:accumulator-granularity} reports
the comparison.

We also evaluate UWSR Expected Fisher, which replaces the realized gradient second
moment with the occurrence-level conditional moment
\(\hat{y}_t(1-\hat{y}_t)\lVert s_{F,i}(x_t)\rVert_2^2/d_F\). A uniform ablation sets
\(\widehat\Sigma_{F,i}=I\).

The penalty requires second-order autodifferentiation through
\(s_{F,i}(x_t)\) (for example, \texttt{create\_graph=True}). The variants
differ only in how they estimate row uncertainty and have nearly identical
computational cost. Their close agreement in Table~\ref{tab:test} shows that
the result is insensitive to choosing the realized or conditional score
moment.

  \subsection{UWSR Training Procedure}
  \label{app:uwsr-algorithm}

  \begin{algorithm}[h]
  \caption{UWSR with a row-wise empirical-Fisher accumulator}
  \label{alg:uwsr}
  \small
  \begin{algorithmic}[1]
  \Require Ordered minibatches $\mathcal D$; epochs $K$; consumer parameters
  $\theta$; embedding tables $\{E_F\}$; penalty strength $\lambda$; shrinkage
  $\gamma_{\mathrm{shr}}$
  \State Initialize $S_{F,i}\gets 0$ for every embedding row
  \For{$q=1,\ldots,K$}
    \For{minibatch $B\in\mathcal D$}
      \State Retrieve occurrence-level lookup outputs
      $a_{b,F,i}=E_F[i]$ for $(F,i)\in\mathcal A(x_b)$
      \State Compute logits $z_b=z_\theta(x_b)$ and
      $L_{\mathrm{BCE}}=|B|^{-1}\sum_{b\in B}\ell(z_b,y_b)$
      \If{$q=1$}
        \ForAll{rows $(F,i)$ activated by $B$}
          \State $g_{B,F,i}\gets\nabla_{E_F[i]}L_{\mathrm{BCE}}$
          \State $S_{F,i}\gets S_{F,i}
          +\operatorname{stopgrad}\!\left(
          \lVert g_{B,F,i}\rVert_2^2/d_F\right)$
        \EndFor
        \State $L\gets L_{\mathrm{BCE}}$
      \Else
        \ForAll{$b\in B$ and $(F,i)\in\mathcal A(x_b)$}
          \State $s_{b,F,i}\gets\nabla_{a_{b,F,i}}z_b$
          \Comment{retain the derivative graph}
        \EndFor
        \State $R_B\gets |B|^{-1}\sum_{b\in B}
        \sum_{(F,i)\in\mathcal A(x_b)}
        u_{F,i}\lVert s_{b,F,i}\rVert_2^2$
        \State $L\gets L_{\mathrm{BCE}}+\lambda R_B$
      \EndIf
      \State Update $\theta$ and $\{E_F\}$ using $\nabla L$
    \EndFor
    \If{$q=1$}
      \ForAll{embedding tables $F$}
        \State $\overline{\mathcal I}_F\gets
        |\mathcal O_F|^{-1}\sum_{i\in\mathcal O_F}S_{F,i}$
        \ForAll{rows $i$ in table $F$}
          \State $\widetilde{\mathcal I}_{F,i}\gets
          (1-\gamma_{\mathrm{shr}})S_{F,i}
          +\gamma_{\mathrm{shr}}\overline{\mathcal I}_F$
          if $i\in\mathcal O_F$
          \State $\widetilde{\mathcal I}_{F,i}\gets
          \overline{\mathcal I}_F$ otherwise
          \State $\rho_{F,i}\gets
          \max\{\widetilde{\mathcal I}_{F,i},
          10^{-6}\overline{\mathcal I}_F\}^{-1}$
        \EndFor
      \EndFor
      \State $c\gets|\mathcal O|^{-1}
      \sum_{(F,i)\in\mathcal O}\rho_{F,i}$
      \State $u_{F,i}\gets\rho_{F,i}/c$ for every row; freeze all $u_{F,i}$
    \EndIf
  \EndFor
  \end{algorithmic}
  \end{algorithm}

\subsection{Alternative \(\Sigma\) constructions}
\label{app:alternatives}

\textbf{Table-Level Shrinkage and Weight Normalization:} For numerical stability during practical experiments, shrinkage was applied. We have not found it sensitive in informal ablations.  Differing non-semantically, the scale of all weight coefficients is normalized such that tuned \(\lambda\) will be closer to 1.

\paragraph{Expected-Fisher variant.}
The expected-Fisher variant accumulates the occurrence-level quantity
$p_t(1-p_t)\lVert s_{F,i}(x_t)\rVert_2^2/d_F$ before row aggregation, where \(p_t\) is the model's prediction for sample \(t\). It uses
the same shrinkage, normalization, and later-pass objective as standard UWSR.

\paragraph{Uniform sensitivity (unweighted ablation).}
$\widehat\Sigma_{F,i}=I$, with any shared scalar absorbed into $\lambda$. This
is the unweighted embedding-row score-sensitivity penalty and carries no row-
or table-specific information. Its squared-derivative form is related to
double backpropagation and input-Jacobian regularization
\citep{drucker1992double}.

\section{Dataset Details}
\label{app:data}

\paragraph{MovieLens-20M.} All 20,000,263 ratings, predicting whether a rating
is at least four stars. The examples are sorted by timestamp and split into
contiguous 80\% training and 5\% validation, calibration, monitoring, and
held-out windows, over
138,493 users and 26,744 items. Inputs are user and item identifiers, 19
genres, release year, training-only support, availability indicators, and 1,128
genome-tag relevance features. In the monitoring window approximately 85\% of
examples have a user unseen in training and 17\% an unseen item.

\paragraph{AmazonElectronics.} The official training file provides
2,608,764 binary examples; the official test file is divided into validation,
calibration, monitoring, and held-out partitions of approximately 96,200 each.
Inputs
are user, target item, target category, and mean-pooled item and category
histories truncated to length 50, over 192,403 users, 63,001 items and 801
categories. Monitoring users and items all occur in training.

\paragraph{Avazu.} All 40,428,967 binary click examples in the RecZoo
preprocessing. The official training file provides 28,300,276 examples; the
labeled evaluation files are concatenated in original order and divided into
validation, calibration, monitoring, and held-out partitions of approximately
3.03 million each. Each of the 22
categorical fields has a separate embedding table, with cardinalities from 5 to
1,048,284, and evaluation-time unseen rates vary by field and reach 34.6\%.

\section{Model, Training and Hyperparameter Details}
\label{app:init} \label{app:training-details}

\paragraph{Model Architecture.}
Sparse tables use 32-dimensional embeddings. The shared consumer is a WuKong-inspired interaction model \citep{zhang2024wukong} containing three residual interaction blocks, four rank-16 factorization heads in each block, followed by a 128-unit SiLU output network.

\paragraph{Model Initialization.} Sparse tables are zero-initialized, with symmetry broken by the linear
projection path of the interaction block; the factorization heads contribute no
embedding gradient at exact zero.

\paragraph{Hyperparameter search.}
The learning rate is $2\times10^{-3}$ and the batch size is 4,096. We select penalty coefficients using the validation partition at \(\sqrt{10}\) interval spacing.

\section{Validation Results and Additional Comparisons}
\label{app:validation-results}

\begin{table}[H]
\centering
\small
\caption{Validation $\Delta$BCE (\%) at epoch four for algorithms compared in Table~\ref{tab:test}. Positive values
denote improvement over the corresponding one-pass model.}
\label{tab:validation-appendix}
\setlength{\tabcolsep}{7pt}
\begin{tabular}{@{}lrrr@{}}
\toprule
Method & MovieLens & Amazon & Avazu \\
\midrule
Naive replay & $-24.66$ & $-101.03$ & $-15.71$ \\
MEDA & $-0.93$ & $+3.37$ & $-0.06$ \\
Freeze embeddings & $-1.89$ & $-1.45$ & $+0.06$ \\
AdamAR & $-0.65$ & $+2.19$ & $+1.96$ \\
UWSR Uniform sensitivity (unweighted ablation) & $+0.40$ & $-1.05$ & $+2.50$ \\
UWSR, rowwise & $+1.47$ & $+6.77$ & $+3.97$ \\
UWSR Expected Fisher, rowwise & $+1.43$ & $+6.95$ & $+4.11$ \\
\bottomrule
\end{tabular}
\end{table}

\begin{table}[t]
\centering
\scriptsize
\caption{Paired test-set changes at epoch four, averaged over five seeds. Brackets give two-sided 95\% $t$-intervals across
paired seed-level differences from the corresponding one-pass Adam checkpoint;
positive values are improvements.}
\label{tab:e4-confidence-intervals}
\setlength{\tabcolsep}{3.2pt}
\begin{tabular}{@{}lrrr@{}}
\toprule
Method & MovieLens & Amazon & Avazu \\
\midrule
\multicolumn{4}{@{}l}{\textit{(a) Relative $\Delta$BCE (\%) with 95\% CI}} \\
Naive replay & $-33.66$ {\scriptsize $[-45.12,-22.20]$} & $-95.43$ {\scriptsize $[-102.76,-88.10]$} & $-13.85$ {\scriptsize $[-15.57,-12.14]$} \\
MEDA & $+0.21$ {\scriptsize $[-2.37,+2.79]$} & $+3.92$ {\scriptsize $[+3.23,+4.62]$} & $-0.11$ {\scriptsize $[-0.58,+0.37]$} \\
Freeze until final pass & $-0.90$ {\scriptsize $[-3.60,+1.81]$} & $-1.74$ {\scriptsize $[-3.07,-0.42]$} & $-0.20$ {\scriptsize $[-0.81,+0.40]$} \\
Active-row AdamW & $+0.10$ {\scriptsize $[-2.45,+2.65]$} & $-23.54$ {\scriptsize $[-25.58,-21.51]$} & $-5.76$ {\scriptsize $[-7.19,-4.34]$} \\
Embedding dropout & $-2.63$ {\scriptsize $[-4.36,-0.90]$} & $-7.16$ {\scriptsize $[-8.82,-5.50]$} & $-7.77$ {\scriptsize $[-10.72,-4.82]$} \\
AdamAR & $-0.86$ {\scriptsize $[-4.81,+3.10]$} & $+2.92$ {\scriptsize $[+2.08,+3.75]$} & $+1.46$ {\scriptsize $[+0.86,+2.05]$} \\
UWSR Uniform & $+1.19$ {\scriptsize $[-1.39,+3.78]$} & $-1.21$ {\scriptsize $[-1.84,-0.58]$} & $+0.48$ {\scriptsize $[-0.19,+1.15]$} \\
UWSR & $+1.38$ {\scriptsize $[-0.89,+3.64]$} & $+6.78$ {\scriptsize $[+6.17,+7.40]$} & $+2.29$ {\scriptsize $[+1.66,+2.92]$} \\
UWSR Expected Fisher & $+1.18$ {\scriptsize $[-1.36,+3.73]$} & $+6.93$ {\scriptsize $[+6.41,+7.46]$} & $+2.49$ {\scriptsize $[+1.90,+3.08]$} \\
\midrule
\multicolumn{4}{@{}l}{\textit{(b) $\Delta$AUC with 95\% CI}} \\
Naive replay & $-0.1669$ {\scriptsize $[-0.1879,-0.1458]$} & $-0.0290$ {\scriptsize $[-0.0323,-0.0258]$} & $-0.0568$ {\scriptsize $[-0.0635,-0.0501]$} \\
MEDA & $-0.0060$ {\scriptsize $[-0.0265,+0.0146]$} & $+0.0126$ {\scriptsize $[+0.0084,+0.0168]$} & $+0.0008$ {\scriptsize $[-0.0032,+0.0048]$} \\
Freeze until final pass & $-0.0207$ {\scriptsize $[-0.0450,+0.0037]$} & $-0.0063$ {\scriptsize $[-0.0112,-0.0015]$} & $-0.0004$ {\scriptsize $[-0.0031,+0.0023]$} \\
Active-row AdamW & $+0.0010$ {\scriptsize $[-0.0201,+0.0222]$} & $-0.0712$ {\scriptsize $[-0.0742,-0.0682]$} & $-0.0214$ {\scriptsize $[-0.0245,-0.0184]$} \\
Embedding dropout & $-0.0407$ {\scriptsize $[-0.0591,-0.0222]$} & $-0.0106$ {\scriptsize $[-0.0124,-0.0089]$} & $-0.0458$ {\scriptsize $[-0.0577,-0.0339]$} \\
AdamAR & $-0.0014$ {\scriptsize $[-0.0193,+0.0165]$} & $+0.0133$ {\scriptsize $[+0.0093,+0.0174]$} & $+0.0096$ {\scriptsize $[+0.0064,+0.0128]$} \\
UWSR Uniform & $+0.0054$ {\scriptsize $[-0.0209,+0.0317]$} & $-0.0007$ {\scriptsize $[-0.0035,+0.0021]$} & $+0.0002$ {\scriptsize $[-0.0044,+0.0047]$} \\
UWSR & $+0.0058$ {\scriptsize $[-0.0165,+0.0281]$} & $+0.0231$ {\scriptsize $[+0.0195,+0.0267]$} & $+0.0141$ {\scriptsize $[+0.0121,+0.0162]$} \\
UWSR Expected Fisher & $+0.0030$ {\scriptsize $[-0.0276,+0.0335]$} & $+0.0232$ {\scriptsize $[+0.0197,+0.0268]$} & $+0.0154$ {\scriptsize $[+0.0126,+0.0182]$} \\
\bottomrule
\end{tabular}
\end{table}

\section{Accumulator granularity ablation}
\label{app:accumulator-granularity}
\begin{table}[h]
\centering
\small
\caption{Accumulator-granularity ablation on Amazon Electronics after four
epochs (seed 7). Both variants selected $\lambda=0.03$ on validation. Auxiliary
state excludes model parameters and optimizer state.}
\label{tab:accumulator-granularity}
\setlength{\tabcolsep}{6pt}
\begin{tabular}{@{}lrrrr@{}}
\toprule
Accumulator & Scalars per row & $\lambda$ & Validation BCE $\downarrow$ & AUC $\uparrow$ \\
\midrule
Row-wise scalar & $1$ & $0.03$ & $\mathbf{0.442510}$ & $\mathbf{0.876256}$ \\
Per-coordinate diagonal & $32$ & $0.03$ & $0.442851$ & $0.875963$ \\
\bottomrule
\end{tabular}
\end{table}

\section{Architecture Results}
\label{app:architecture-results}

\begin{table}[H]
\centering
\small
\caption{Amazon validation results across WuKong, DCNv2
\citep{wang2021dcnv2}, and AutoInt \citep{song2019autoint}. Entries are
$\Delta$BCE relative to each architecture's one-pass baseline.}
\label{tab:arch}
\setlength{\tabcolsep}{3.5pt}
\begin{tabular}{@{}lrrrrrrrr@{}}
\toprule
& One pass & \multicolumn{6}{c}{$\Delta$BCE} & Fisher \\
\cmidrule(lr){3-8}
Architecture & BCE & Naive & MEDA & Freeze & AdamAR & Isotropic & UWSR & $\Delta$AUC \\
\midrule
WuKong & 0.474634 & $-101.03$ & $+3.37$ & $-1.45$ & $+2.19$ & $-1.05$ & $\mathbf{+6.95}$ & $+0.0222$ \\
DCNv2 & 0.507843 & $-45.23$ & $+6.71$ & $+3.08$ & $+5.90$ & $+2.26$ & $\mathbf{+9.02}$ & $+0.0362$ \\
AutoInt & 0.491371 & $-57.86$ & $+5.29$ & $-3.35$ & $+5.56$ & $-0.97$ & $\mathbf{+8.58}$ & $+0.0303$ \\
\bottomrule
\end{tabular}
\end{table}

\section{Computational Cost}
\label{app:compute}

\begin{table}[H]
\centering
\small
\caption{Amazon computational overhead relative to naive replay four-pass training, mean 2 seeds.
Later-epoch step time is relative to one naive replay step; end-to-end
overhead includes all epochs and method-specific setup.}
\label{tab:compute}
\setlength{\tabcolsep}{7pt}
\begin{tabular}{@{}lrrr@{}}
\toprule
Method & Epochs & Later-epoch step time & End-to-end overhead \\
\midrule
Naive replay / MEDA / AdamAR & 4 & $1.00\times$ & $1.00\times$ \\
UWSR & 4 & $1.18\times$ & $1.13\times$ \\
\bottomrule
\end{tabular}
\end{table}

The paired arms ran sequentially on otherwise idle GPUs with diagnostics and candidate caching disabled. Per-seed later-pass ratios are 1.16 and 1.20; end-to-end ratios are 1.12 and 1.14.

\end{document}